%% file: online-broadcast-arxiv.tex
\documentclass[]{article}
\usepackage{authblk}

\usepackage[utf8]{inputenc} 
\usepackage[T1]{fontenc}    
\usepackage{hyperref}       
\usepackage{url}            
\usepackage{booktabs}       
\usepackage{amsfonts}       
\usepackage{nicefrac}       
\usepackage{microtype}      
\usepackage{fullpage}

\usepackage{adjustbox}
\usepackage{hyperref}
\usepackage{url}
\usepackage{graphicx}
\usepackage{amsmath}
\usepackage{bbm}
\usepackage{paralist}
\usepackage{Definitions}
\usepackage[ruled,noline]{algorithm2e} 
\usepackage{color}
\usepackage{enumitem}
\usepackage{mathtools}
\usepackage{subfig}
\usepackage{graphicx}
\usepackage{tabularx}
\usepackage[dvipsnames]{xcolor}

\newcommand{\redqueen}{{\textsc{Red\-Queen}}\xspace}

\newcommand{\xhdr}[1]{\vspace{0.1mm}\noindent{{\bf #1. }}}

\newcommand{\qb}{\boldsymbol{\omega}}

\newcommand\blfootnote[1]{%
  \begingroup
  \renewcommand\thefootnote{}\footnote{#1}%
  \addtocounter{footnote}{-1}%
  \endgroup
}

 \graphicspath{{./figures/}}

\SetKwProg{Fn}{Function}{}{end}

\title{\redqueen: An Online Algorithm for \\ Smart Broadcasting in Social Networks}

\author[1]{Ali Zarezade$^{*}$}
\author[2]{Utkarsh Upadhyay$^{*}$}
\author[1]{Hamid R. Rabiee}
\author[2]{Manuel Gomez-Rodriguez}

\affil[1]{Sharif University, zarezade@ce.sharif.edu, rabiee@sharif.edu}
\affil[2]{Max Planck Institute for Software Systems, utkarshu@mpi-sws.org, manuelgr@mpi-sws.org}

\date{}

\begin{document}


\maketitle

\begin{abstract}
\input{000abstract}
\end{abstract}

\blfootnote{$^{*}$\scriptsize Authors contributed equally. This work was done during Ali Zarezade'{}s internship at Max Planck
Institute for Software Systems.}

\section{Introduction}
\label{sec:introduction}
\input{010introduction}

\section{Preliminaries}
\label{sec:background}
\input{020background}

\section{Problem Formulation}
\label{sec:formulation}
\input{030formulation}

\section{Stochastic Optimal Control Algorithm}
\label{sec:method}
\input{040method}

\section{Experiments}
\label{sec:experiments}
\input{050experiments}

\section{Conclusions}
\label{sec:conclusions}
\input{060conclusions}

{
\bibliographystyle{abbrv}
\bibliography{refs}
}

\newpage

\appendix
\input{070appendix}

\end{document}

%% file: 000abstract.tex
%
Users in social networks whose posts stay at the top of their followers'{} feeds the longest time are more likely to
be \emph{noticed}.
Can we design an online algorithm to help them decide when to post to stay at the top?
In this paper, we address this question as a novel optimal control problem for jump stochastic differential equations.
For a wide variety of feed dynamics, we show that the optimal broad\-cas\-ting intensity for any user is surprisingly
simple -- it is given by the position of her most recent post on each of her follower'{}s feeds.
As a consequence, we are able to develop a simple and highly efficient online algorithm, \redqueen, to sample the
optimal times for the user to post.
Experiments on both synthetic and real data gathered from Twitter show that our algorithm is able to consistently
make a user'{}s posts more visible over time, is robust to volume changes on her followers'{} feeds, and significantly
outperforms the state of the art.

%% file: 010introduction.tex
Whenever a user in an online social network decides to share a new story with her followers, she is often competing for \emph{attention} with dozens, if not hundreds, of
stories \emph{simultaneously} shared by other users that their followers follow~\cite{backstrom2011center,gomez14icwsm}.
In this context, recent empirical studies have shown that stories at the top of their followers'{} feed are more likely to be \emph{noticed} and consequently liked or
shared~\cite{hodas2012visibility,kang2015vip,lerman2014leveraging}.
Can we find an algorithm that helps a user decide when to post to increase her chances to stay at the top?

The ``when-to-post'' problem was first studied by Spasojevic et al.~\cite{spasojevic2015post}, who performed a large empirical study on the best times to post in Twitter and Facebook, measuring attention a user elicits by means of the number of responses to her posts. Moreover, they designed several heuristics to pinpoint at the times that elicit the greatest attention 
in a training set and showed that these times also lead to more responses in a held-out set.
Since then, algorithmic approaches to the ``when-to-post'' problem with provable guarantees have been largely lacking. Only very recently, Karimi et al.~\cite{karimi2016smart} introduced a convex optimization framework to find optimal broadcasting strategies, measuring attention a user elicits as the time that at least one of her posts is among the 
$k$ most recent stories received in her followers'{} feed.
However, their algorithm requires expensive data pre-processing, it does not adapt to changes in the users'{} feeds dynamics, and in practice, it is less effective than our proposed algorithm,
as shown in Section~\ref{sec:experiments}.

In this paper, we design a novel online algorithm for the when-to-post problem, where we measure \emph{visibility} of a broadcaster as the position of her most recent post
on her followers'{} feeds over time.
A desirable property of this visibility measure is that it can be easily extracted from real data without actual interventions --- given any particular broadcasting strategy for a user, one
can always measure its visibility using a separate held-out set of the user'{}s followers'{} feeds~\cite{karimi2016smart}.
In contrast, measures based on users'{} reactions (\eg, number of likes, shares and replies) are difficult to estimate from real data, due to the
presence of other confounding factors such as users'{} influence, content and wording~\cite{cheng2014can, chenhao14, lakkaraju2013s}.

More precisely, we represent users'{} posts and feeds using the framework of temporal point processes, which characterizes the continuous time interval between posts
using conditional intensity functions~\cite{AalBorGje08}. Under this representation, finding the optimal broadcasting (or posting) strategy for a user reduces to finding its
associated conditional broadcasting intensity~\cite{karimi2016smart}.
Then, for a large family of intensity functions, which includes Hawkes~\cite{hawkes1971spectra} and Poisson~\cite{Kingman1992} as particular instances, we find ``when-to-post''
by solving a novel optimal control problem for a system of jump stochastic differential equations (SDEs)~\cite{hanson2007}. Our problem formulation differs from previous literature
in two key technical aspects, which are of independent interest:
\begin{itemize}[noitemsep,nolistsep,leftmargin=1cm]
\item[I.] The control signal is a conditional (broadcasting) intensity, which is used to sample stochastic events (\ie, stories to post). As a consequence, the problem formulation requires
another layer of stochasticity. In previous work, the control signal is a time-varying real vector.
\item[II.] The (broadcasting) intensities are stochastic Markov processes and thus the dynamics are doubly stochastic.
This requires us to redefine the cost-to-go to incorporate the instantaneous value of these intensities as  additional arguments. Previous work has typically considered constant intensities
and only very recently time-varying deterministic intensities~\cite{wang2016steering}.
\end{itemize}
These technical aspects have implications beyond the smart broadcasting problem since they enable us to establish a previously unexplored
connection between optimal control of jump SDEs and double stochastic temporal point processes (\eg, Hawkes processes), which have been
increasingly used to model social activity~\cite{shaping14nips,coevolve15nips,zhao2015seismic}.

Moreover, we find that the solution to the above optimal control problem is surprisingly simple: the optimal broadcasting intensity for a user is given by the position of her most
recent post on each of her follower'{}s feeds.
This solution allows for a simple and highly efficient online procedure to sample the optimal times for a user to broadcast, which can be implemented in a few
lines of code and does not require fitting a model for the feeds'{} intensities.
Finally, we performed experiments on both synthetic and real data gathered from Twitter and show that our algorithm is able to consistently make a user'{}s posts stay
at the top of her followers'{} feeds, is robust to changes on the dynamics (or volume) of her followers'{} feeds, and significantly outperforms the state of the art~\cite{karimi2016smart}.

\xhdr{Further related work} In addition to the paucity of work on the when-to-post problem~\cite{karimi2016smart,spasojevic2015post}, discussed previously, our work
also relates to:
(i) empirical studies on attention and information overload in social and information networks~\cite{backstrom2011center,gomez14icwsm,hodas2012visibility,miritello2013limited},
which investigate whether there is a limit on the amount of ties (\eg, friends, followees or phone contacts) people can maintain, how people distribute attention across them, and how attention influences the propagation
of information;
and (ii) the influence maximization problem~\cite{CheWanWan2010,du13nips,shaping14nips,KemKleTar03,RicDom02}, which aims to find a set of nodes in a social network whose initial adoption of a certain idea or product can trigger the
largest expected number of follow-ups.
In contrast, we focus on optimizing a social media user'{}s broadcasting strategy
to capture the greatest attention from the followers.

%% file: 020background.tex
We first revisit the framework of temporal point processes~\cite{AalBorGje08} and then use it to represent broadcasters and feeds in social and information
networks.

\xhdr{Temporal point processes}
A temporal point process is a stochastic process whose realization consists of a sequence of discrete events localized in time, $\Hcal = \{t_i\in \RR^{+} \,\vert\, i\in \mathbb{N}^+,\, t_i<t_{i+1} \}$.
In recent years, they have been used to represent many different types of event data produced in online social networks and the web, such as the times of
tweets~\cite{shaping14nips}, retweets~\cite{zhao2015seismic}, or links~\cite{coevolve15nips}.

A temporal point process can also be represented as a counting process $N(t)$, which is the number of events up to time $t$. Moreover, given $\Hcal(t)=\{ t_i \in \Hcal \,\vert\, t_i < t \}$, the history of event times  up to but not including time $t$, we can characterize the counting process using the conditional intensity function $\lambda^{*}(t)$, which is the conditional
probability of observing an event in an infinitesimal window $[t, t + dt)$ given the history $\Hcal(t)$, \ie,
\begin{equation}
\lambda^{*}(t)dt = \PP\cbr{\text{event in $[t, t+dt) | \Hcal(t)$}} = \EE[dN(t)|\Hcal(t)], \nonumber
\end{equation}
where $dN(t) \in \{0, 1\}$ 
and the sign $^{*}$ means that the intensity may depend on the history $\Hcal(t)$.

The functional form for the intensity is often chosen to capture the phenomena of interest.
For example, in the context of modeling social activity, retweets have been modeled using Hawkes processes~\cite{shaping14nips,zhao2015seismic} and daily and weekly variations
on the volume of posted tweets have been captured using Poisson processes~\cite{navaroli2015modeling,karimi2016smart}.
In this work, we consider the following general functional form, which includes Hawkes and Poisson as particular instances:
\begin{equation} \label{eq:intensity}
\lambda^{*}(t) = \lambda_0(t) + \alpha \int_0^t g(t-s) dN(s),
\end{equation}
where $\lambda_0(t) \geq 0$ is a time-varying function, which models the publication of messages by users on their own initiative, the second
term, with $\alpha \geq 0$, models the publication of additional messages (\eg, replies, shares) by the users due to the influence that previous messages
(their own as well as the ones posted by others) have on their intensity, and $g(t)$ denotes an exponential triggering kernel $e^{-w t} \II(t \geq 0)$. 
The second term makes the intensity dependent on history and a stochastic process by itself.
Finally, the following alternative representation will be useful to design our stochastic optimal control algorithm for smart broadcasting (proven in the Appendix):
\begin{proposition}\label{thm:jsde-intensity}
	Let $N(t)$ be a counting process with an associated intensity $\lambda^{*}(t)$ given by Eq.~\ref{eq:intensity}. Then the tuple $(N(t), \lambda^{*}(t))$ is a doubly stochastic Markov process, whose dynamics can be defined by the following jump SDE:
	\begin{align}\label{eq:jsde-intensity}
		d\lambda^{*}(t) = \left[\lambda_0'(t)+w \lambda_0(t) -w \lambda^{*}(t)\right]dt + \alpha dN(t),
	\end{align}
	with initial condition $\lambda^{*}(0) = \lambda_0(0)$.
\end{proposition}
In the remainder of the paper, to simplify the notation, we drop the sign $^{*}$ from the intensities.

\xhdr{Representation of broadcasters and feeds}
Given a directed network $\mathcal{G}=(\mathcal{V},\mathcal{E})$ with $|\mathcal{V}|=n$ users, we assume any user can be a broadcaster, a follower or both, each broadcaster can be followed by
multiple followers, and each follower can follow multiple broadcasters.
Then, we represent the broadcasting times of the users as a collection of counting processes denoted by a vector $\bm{N}(t)$, in which the $i$-th dimension, $N_{i}(t)$, is the number
of messages or \emph{stories} broadcasted by user $i$ up to time $t$. Here, we denote the history of times of the stories broadcasted by user $i$ by time $t$ as $\Hcal_i(t)$, the entire history of
times as $\Hcal(t) = \cup_{i \in \Vcal} \Hcal_i(t)$, and characterize these counting processes using their corresponding intensities, \ie, $\mathbb{E}[d\bm{N}(t) | \Hcal(t)] = \bm{\mu}(t) \, dt$.

Given the adjacency matrix $A\in \{0,1\}^{n\times n}$, where $A_{ij}=1$ indicates that user $j$ follows user $i$, we can represent the times of the stories users receive in
their feeds from the broadcasters they follow as a sum of counting processes, $A^T \bm{N}(t)$, and calculate the corresponding conditional intensities as $\gammab(t) = A^T \bm{\mu}(t)$.
Here, we denote the history of times of the stories received by user $j$ by time $t$ as $\Fcal_j(t) := \cup_{i \in \Ncal(j)} \Hcal_{i}(t)$, where $\Ncal(j)$ is the set of users that $j$
follows.

Finally, from the perspective of a broadcaster $i$, it is useful to define the counting processes $\bm{M}_{{\scriptscriptstyle\setminus} i}(t) =  A^T \bm{N}(t) - A_i N_i(t)$, in which the $j$-th
dimension, $M_{j {\scriptscriptstyle\setminus} i}(t)$, represents the times of the stories user $j$ receives due to other broadcasters she follows, and $A_i$ is the $i$-th row of the
adjacency matrix $A$.
Moreover, for each of these counting processes, the conditional intensity is given by $\gamma_{j {\scriptscriptstyle\setminus} i}(t) = \gamma_{j}(t) - \mu_i(t)$ and the history
is given by $\Fcal_{j {\scriptscriptstyle\setminus} i}(t) := \Fcal_j(t) \backslash \Hcal_{i}(t)$.
%
%

%% file: 030formulation.tex
In this section, we first define our visibility measure, $r(t)$, then derive a jump stochastic differential equation that links our measure to the counting processes
associated to a broadcaster and her followers, and conclude with a statement of the when-to-post problem for our visibility measure.
\begin{figure}
\centering
\includegraphics[width=0.7\textwidth]{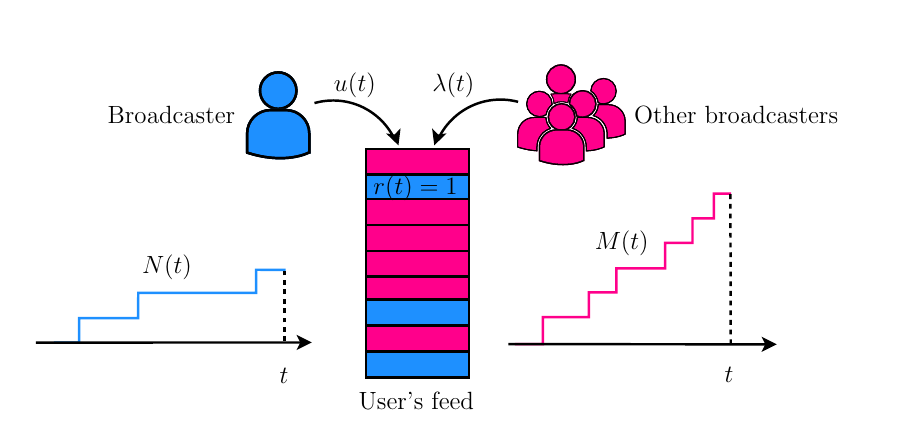}
\caption{The dynamics of visibility. A broadcaster $i$ posts $N_i(t) = N(t)$ messages with intensity $\mu_{i}(t) = u(t)$. Her messages accumulate in her follower $j$'{}s 
feed, competing for attention with $M_{j\backslash i}(t) = M(t)$ other messages, posted by other broadcasters $j$ follows with intensity $\gamma_{j \backslash i}(t) = \lambda(t)$. 
The visibility function $r_{ij}(t) = r(t)$ is the position or rank of the most recent story posted by broadcaster $i$ in the follower $j$'{}s feed by time t.}
\label{fig:model}
\vspace{-2mm}
\end{figure}

\xhdr{Definition of visibility}
Given a broadcaster $i$ and one of her followers $j$, we define the visibility function $r_{ij}(t)$ as the position or \emph{rank} of the most recent story posted
by $i$ in $j$'{}s feed by time $t$, which clearly depends on the feed ranking mechanism in the corresponding social network.
Here, for simplicity, we assume each user'{}s feed ranks stories in inverse chro\-no\-lo\-gi\-cal order\footnote{\scriptsize At the time of writing, Twitter
and Weibo rank stories in inverse chronological order by default and Facebook allows choosing such an ordering.}.
However, our framework can be easily extended to any feed ranking mechanisms, as long as its rank dynamics can be expressed as a
jump SDE\footnote{\scriptsize This would require either having access to the corresponding feed ranking mechanism or reverse engineering it, which is out of the scope of this work.}.

Under the inverse chronological ordering assumption, position is simply the number of stories that others broadcasters posted in $j$'s feed from the time of the most recent story posted
by $i$ until $t$.
Then, when a new story arrives to a user'{}s feed, it appears at the top of the feed and the other stories are shifted down by one.
If we identify the time of the most recent message posted by $i$ by time $t$ as $\tau_{i}(t) = \max\{t_k \in \Hcal_i(t)\}$, then the visibility is formally defined as:
\begin{align}\label{equ:rank}
	r_{ij}(t) = M_{j {\scriptscriptstyle\setminus} i}(t) - M_{j {\scriptscriptstyle\setminus} i}(\tau_{i}(t)),
\end{align}
Note that, if the last story posted by $i$ is at the top of $j$'{}s feed at time $t$,
then $r_{ij}(t) = 0$.

\xhdr{Dynamics of visibility}
%
Given a broadcaster $i$ with broadcasting counting process $N_i(t)$ and one of her followers $j$ with feed counting process due to other broadcasters $M_{j {\scriptscriptstyle\setminus} i}(t)$, the rank of $i$ in $j$'s feed $r_{ij}(t)$ satisfies the following equation:
\begin{equation}
	r_{ij}(t+dt) =
	\underset{\text{\scriptsize 1. Increases by one}}{\underbrace{(r_{ij}(t)+1) dM_{j {\scriptscriptstyle\setminus} i}(t) (1-dN_i(t))}}
	+\underset{\text{\scriptsize 2. Becomes zero}}{\underbrace{0}} \nonumber +
	\underset{\text{\scriptsize 3. Remains the same}}{\underbrace{r_{ij}(t)(1 - dM_{j {\scriptscriptstyle\setminus} i}(t))(1- dN_i(t))}}, \nonumber
\end{equation}
where each term models one of the three possible situations:
\begin{itemize}[noitemsep,nolistsep,leftmargin=1cm]
\item[1.] The other broadcasters post a story in $(t,t+dt]$, $dM_{j {\scriptscriptstyle\setminus}i}(t)=1$, and broadcaster $i$ does not post, $dN_i(t)=0$. The position of the last story posted by
$i$ in $j$'{}s feed steps down by one, \ie, $r_{ij}(t+dt)=r_{ij}(t)+1$.
\item[2.] Broadcaster $i$ posts a story in $(t,t+dt]$, $dN_i(t)=1$, and the other broadcasters do not, $dM_{j {\scriptscriptstyle\setminus}i}(t)=0$. No matter what the previous rank was, the new
rank is $r_{ij}(t+dt)=0$ since the newly posted story appears at the top of $j$'{}s feed.
\item[3.] No one posts any story in $(t,t+dt]$, $dN_i(t)=0$ and $dM_{j {\scriptscriptstyle\setminus}i}(t)=0$. The rank remains the same, \ie, $r_{ij}(t+dt)=r_{ij}(t)$
\end{itemize}
We skip the case in which $M_{j {\scriptscriptstyle\setminus}i}(t)=1$ and $dN_i(t)=1$ in the same time interval $(t,t+dt]$ because, by the Blumenthal zero-one law~\cite{blumenthal1957extended},
it has zero probability. Now, by rearranging terms and using that $dN_i(t) dM_{j {\scriptscriptstyle\setminus}i}(t)=0$, we uncover the following jump SDE for the visibility
(or rank) dynamics:
\begin{align}\label{eq:rank-dynamics}
dr_{ij}(t) &= -r_{ij}(t)\,dN_i(t) + dM_{j {\scriptscriptstyle\setminus}i}(t).
\end{align}
where $dr_{ij}(t)=r_{ij}(t+dt)-r_{ij}(t)$. Figure~\ref{fig:model} illustrates the concept of visibility for one broadcaster and one follower.

\xhdr{The when-to-post problem}
Given a broadcaster $i$ and her followers $\Ncal(i)$, our goal is to find the optimal conditional intensity $\mu_i(t) = u(t)$ that minimizes the expected value of a 
%
particular nondecreasing convex loss function $\ell(\bm{r}(t), u(t))$ of the broadcaster'{}s visibility on each of her follower'{}s feed, $\bm{r}(t) = [r_{ij}(t)]_{j \in \Ncal(i)}$,
and the intensity itself, $u(t)$, over a time window $(t_0, t_f]$, \ie, 
%
\begin{align} \label{eq:average-loss}
\underset{u(t_0, t_f]}{\text{minimize}} & \quad \EE_{(N_i,\bm{M}_{{\scriptscriptstyle\setminus}i})(t_0, t_f]}\left[ \phi(\bm{r}(t_f)) + \int_{t_0}^{t_f} \ell(\bm{r}(\tau), u(\tau)) d\tau\right] \nonumber \\
\text{subject to} & \quad u(t) \geq 0 \quad \forall t \in (t_0, t_f],
\end{align}
where $u(t_0, t_f]$ denotes user $i$'{}s intensity from $t_0$ to $t_f$, the expectation is taken over all possible realizations of the counting processes associated to user $i$ and all
other broadcasters from $t_0$ to $t_f$, denoted as $(N_i,\bm{M}_{{\scriptscriptstyle\setminus}i})(t_0, t_f]$, and $\phi(\bm{r}(t_f))$ is an arbitrary penalty function\footnote{\scriptsize The
final penalty function $\phi(\bm{r}(t_f))$ is necessary to derive the optimal intensity $u^*(t)$ in Section~\ref{sec:method}. However, the actual optimal intensity $u^*(t)$ does not
depend on the particular choice of terminal condition.}.
%
%
%
%
Here, by considering a nondecreasing loss, we penalize times when the position of the most recent story on each of the follower'{}s feeds is high (\ie, the most recent story does not stay
\emph{at the top}) and we limit the number of stories the broadcaster can post.
Finally, note that the optimal intensity $u(t)$ for broadcaster $i$ at time $t$ may depend on the visibility $\bm{r}(t)$ with respect to each
of her followers and thus the associated counting process $N_i(t)$ may be doubly stochastic.

%% file: 040method.tex
In this section, we tackle the when-to-post problem defined by Eq.~\ref{eq:average-loss} from the perspective of stochastic optimal control of jump SDEs~\cite{hanson2007}. More specifically, we first derive a solution to the problem considering only one follower, provide an efficient practical implementation of the solution and then generalize it to the case of multiple followers.
We conclude this section by deriving a solution to the problem given an (idealized) oracle that knows the times of all stories in the
followers'{} feeds \emph{a priori}, which we will use as baseline.

\xhdr{Optimizing for one follower} Given a broadcaster $i$ with $N_i(t) = N(t)$ and $\mu_i(t) = u(t)$ and only one of her followers $j$ with $M_{j {\scriptscriptstyle\setminus} i}(t) = M(t)$ and
$\gamma_{j {\scriptscriptstyle\setminus} i}(t)=\lambda(t)$, we can rewrite the when-to-post problem defined by Eq.~\ref{eq:average-loss} as
\begin{align}
\underset{u(t_0, t_f]}{\text{minimize}} & \quad \EE_{(N,M)(t_0, t_f]}\left[ \phi(r(t_f)) + 	\int_{t_0}^{t_f} \ell(r(\tau),u(\tau)) \, d\tau \right] \nonumber \\
\text{subject to} & \quad u(t) \geq 0 \quad \forall t \in (t_0, t_f], \label{eq:average-loss-one-follower}
\end{align}
where, using Eq.~\ref{eq:jsde-intensity} and Eq.~\ref{eq:rank-dynamics}, the dynamics of $M(t)$ and $r(t)$ are given by the following two coupled jump SDEs:
\begin{align}
	dr(t) &= -r(t)\,dN(t) + dM(t) \nonumber \\
	d\lambda(t) &= \left[\lambda_0'(t)+w \lambda_0(t) -w \lambda(t)\right]dt + \alpha \, dM(t), \nonumber
\end{align}
with initial conditions $r(t_0)=r_0$ and $\lambda(t_0)=\lambda_0$, and the dynamics of $N(t)$ are given by the intensity $u(t)$ that we aim to optimize.
The above stochastic optimal control problem differs from previous literature in two key technical aspects, which require careful reasoning:
\begin{itemize}[noitemsep,nolistsep,leftmargin=1cm]
\item[(i)] The control signal $u(t)$ is a conditional intensity, which controls the dynamics of the counting process $N(t)$ (\ie, number of stories broadcasted by user $i$ by time $t$). As a consequence,
the problem formulation needs to account for another layer of stochasticity. Previous work assumes the control signal to be a time-varying real vector.
\item[(ii)] The dynamics of the counting process $M(t)$ (\ie, number of stories broadcasted by other users that $j$ follows by time $t$) are doubly stochastic Markov.
Previous work has typically considered memoryless Poisson processes and, only very recently, inhomogeneous Poisson~\cite{wang2016steering}.
\end{itemize}

Next, we will define a novel optimal cost-to-go function that accounts for the above unique aspects of our problem, showing that the Bellman'{}s principle of optimality still follows, and finally find
the optimal solution using the corresponding Hamilton-Jacobi-Bellman (HJB) equation.
\begin{definition}\label{thm:cost-def}
The optimal cost-to-go $J(r(t),\lambda(t),t)$ is defined as the minimum of the expected value of the cost of going from state $r(t)$ with intensity $\lambda(t)$ at time $t$ to final state at time $t_f$, \ie,
\begin{equation}\label{eq:cost-def}
	\min_{u(t,t_f]} \mathbb{E}_{(N,M)(t,t_f]} \left[ \phi(r(t_f)) + \int_t^{t_f} \ell(r(\tau),u(\tau)) \, d\tau \right ],
\end{equation}
where the expectation is taken over all trajectories of the control and noise jump process, $N$ and $M$, in the $(t,t_f]$ interval, given the initial values of $r(t)$, $\lambda(t)$
and $u(t)$.
\end{definition}
To find the optimal control $u(t, t_f]$ and cost-to-go $J$, we break the problem into smaller subproblems, using the Bellman'{}s principle of optimality, which the above definition
allows (proven in Appendix):
\begin{lemma}[Bellman'{}s Principle of Optimality] \label{thm:bellman}
The optimal cost satisfies the following recursive equation:
\begin{equation}
	J(r(t),\lambda(t),t) =
	\min_{u(t,t+dt]} \mathbb{E} \left[ J(r(t+dt),\lambda(t+dt),t+dt)\right ] + \ell(r(t),u(t)) \, dt. \label{eq:bellman}
\end{equation}
\end{lemma}
where the expectation is taken over all trajectories of the control and noise jump processes, $N$ and $M$, in $(t, t+dt]$.
Then, we use the Bellman'{}s principle of optimality to derive a partial differential equation on $J$, often called the Hamilton-Jacobi-Bellman (HJB)
equation~\cite{hanson2007}.
To do so, we first assume $J$ is continuous and then rewrite Eq.~\ref{eq:bellman} as
\begin{align}
	J(r(t),\lambda(t),t) &= \min_{u(t,t+dt]} \mathbb{E} \left[ J(r(t),\lambda(t),t) + dJ(r(t),\lambda(t),t) \right ] +\ell(r(t),u(t)) \, dt \nonumber \\ 
	0 = \min_{u(t,t+dt]} & \mathbb{E} \left[dJ(r(t),\lambda(t),t) \right ] + \ell(r(t),u(t)) \, dt. \label{eq:diff-j-eq-short} 
\end{align}

Then, we differentiate $J$ with respect to time $t$, $r(t)$ and $\lambda(t)$ using Lemma~\ref{thm:differentialSDE} (refer to
Appendix).
\begin{algorithm}[t]
\small
\DontPrintSemicolon 
\KwIn{Parameters $q$ and $s$}
\KwOut{Returns time for the next post}
$t \gets \infty$;\, $\tau \gets othersNextPost(\,)$ \;
\While{$\tau < t$} {
  $\Delta \sim \exp(\sqrt{{s}/{q}}) $\;
  $t \gets \min(t,\,\tau+\Delta)$\;
  $\tau \gets othersNextPost(\,)$\;
}
\Return $t$\;
\caption{\redqueen{} for fixed $s$, $q$ and one follower.}
\label{alg:sampling}
\end{algorithm}
Specifically, consider $x(t) = r(t)$, $y(t) = \lambda(t)$ and $F = J$ in the above mentioned lemma, then,
\begin{align}
	dJ(r(t),\lambda(t),t) &= J_t(r(t),\lambda(t),t)dt  + \left[\lambda_0'(t) + w \lambda_0(t)-w\lambda(t) \right]J_{\lambda}(r(t),\lambda(t),t)dt \nonumber \\*
	&+[J(0,\lambda(t),t)-J(r(t),\lambda(t),t)] dN(t) + [J(r(t)+1,\lambda(t)+\alpha,t)-J(r(t),\lambda(t),t)] dM(t) \nonumber
\end{align}

Next, if we plug in the above equation in Eq.~\ref{eq:diff-j-eq-short}, it follows that
\begin{align}
	0 &= \min_{u(t,t+dt]} \Big\{J_t(r(t),\lambda(t),t) dt +\left[\lambda_0'(t)+ w\lambda_0(t)-w\lambda(t) \right]J_{\lambda}(r(t),\lambda(t),t) dt \nonumber \\*
	&+[J(0,\lambda(t),t)-J(r(t),\lambda(t),t)] \mathbb{E}[dN(t)] + [J(r(t)+1,\lambda(t)+\alpha,t)-J(r(t),\lambda(t),t)] \mathbb{E}[dM(t)] + \ell(r(t),u(t)) \, dt \Big\}.  
\end{align}
Now, using $\mathbb{E}[dN(t)] = u(t)dt$ and $\mathbb{E}[dM(t)] = \lambda(t) dt$, and rearranging terms, the HJB equation follows:
\begin{align}
	0&= J_t(r(t),\lambda(t),t) +\left[\lambda_0'(t)+ w\lambda_0(t)-w\lambda(t) \right]J_{\lambda}(r(t),\lambda(t),t)+[J(r(t)+1,\lambda(t)+\alpha,t)-J(r(t),\lambda(t),t)]\lambda(t) \nonumber \\*
	&+\min_{u(t,t+dt]} \ell(r(t),u(t)) + [J(0,\lambda(t),t)-J(r(t),\lambda(t),t)]u(t).  \label{eq:bellman-pde-min} 
\end{align}
To be able to continue further, we need to define the loss $\ell$ and the penalty $\phi$. Following the literature on stochastic optimal
control~\cite{hanson2007}, we consider the following quadratic forms, which will turn out to be a tractable choice\footnote{\scriptsize
Considering other losses with a specific semantic meaning (\eg, $\II(r(t) \leq k)$) is a challenging direction for future work.}:
\begin{equation}
	\phi(r(t_f)) = \frac{1}{2}r^2(t_f) \quad \text{and} \quad \ell(r(t),u(t)) = \frac{1}{2} s(t) \,r^2(t) + \frac{1}{2} q\,u^2(t), \nonumber
\end{equation}
where $s(t)$ is a time significance function $s(t) \geq 0$, which favors some periods of times (\eg, times in which the follower is online\footnote{\scriptsize
Such information may be hidden but one can use the followers'{} posting activity or geographic location as a proxy~\cite{karimi2016smart}.}),
and $q$ is a given parameter, which trade-offs visibility and number of broadcasted posts.
\begin{algorithm}[t]
\small
\DontPrintSemicolon 
\KwIn{Initial state $r_0$, interval widths $w_1,\ldots,w_{m+1}$, parameter $q$ and significance $s(t) = s$}
\KwOut{Overall cost $J(r_0,0)$, optimal control $u_0^*,\ldots,u_m^*$}
\For{$r \gets r_0+m$ \KwTo $0$}{
$J(r,m+1) \gets \frac{1}{2}r^2$\;
}
\For{$k \gets m$ \KwTo $0$}{
\For{$r \gets r_0+k-1$ \KwTo $0$}{
$J(r,k)=\min\{\frac{1}{2}q + J(0,k+1)  ,\, \frac{1}{2} s w_{k+1}(r+1)^2 + J(r+1,k+1)\}$\;
}
}
\For{$k \gets 0$ \KwTo $m$}{
\eIf{$\frac{1}{2}q + J(0,k+1) < \frac{1}{2} s w_{k+1}(r_k+1)^2 + J(r_k+1,k+1)$}{\
  $u_k^* \gets 1;\, r_{k+1} \gets 0$\;
}{
  $u_k^* \gets 0;\, r_{k+1} \gets r_k+1$\;
}
}
\Return{$J(r_0,0),\, u_0^*,\ldots,u_m^*$}
\caption{Optimal posting times with an oracle.}
\label{alg:oracle}
\end{algorithm}

Under these definitions, we take the derivative with respect to $u(t)$ of Eq.~\ref{eq:bellman-pde-min} and
uncover the relationship between the optimal intensity and the optimal cost:
\begin{align} \label{eq:relationship-rate-j}
	u^*(t) = q^{-1}\left[J(r(t),\lambda(t),t) - J(0,\lambda(t),t)\right].
\end{align}
Finally, we substitute the above expression in Eq.~\ref{eq:bellman-pde-min} and find that the optimal cost $J$ needs to satisfy the following
nonlinear differential equation:
\begin{align}
	0&=J_t(r(t),\lambda(t),t) + \left[\lambda_0'(t)+w\lambda_0(t)-w\lambda(t)\right]J_{\lambda}(r(t),\lambda(t),t) +[J(r(t)+1,\lambda(t)+\alpha,t)-J(r(t),\lambda(t),t)]\lambda(t) \nonumber\\
	&+\frac{1}{2}s(t)\,r^2(t)-\frac{1}{2}q^{-1}\left[J(r(t),\lambda(t),t)-J(0,\lambda(t),t)\right]^2 \label{eq:bellman-pde}
\end{align}
with $J(r(t_f),\lambda(t_f),t_f)=\phi(r(t_f))$ as the terminal condition. The following lemma provides us with a solution to the above
equation (proven in Appendix):
\begin{lemma} \label{lem:opt-con-sol}
Any solution to the nonlinear differential equation given by Eq.~\ref{eq:bellman-pde} can be approximated as closely
as desired by
\begin{align*}
	J(r(t),\lambda(t),t) = f(t)+ \sqrt{{s(t)}/{q}} \,  r(t) + \sum_{j=1}^m g_{j}(t)\lambda^j(t),
\end{align*}
where $f(t)$ and $g_{j}(t)$ are time-varying functions, and $m$ controls for the approximation guarantee.
\end{lemma}
Given the above Lemma and Eq.~\ref{eq:relationship-rate-j}, the optimal intensity is readily given by
following theorem:
\begin{theorem}
The optimal intensity for the when-to-post problem defined by Eq.~\ref{eq:average-loss-one-follower} with quadratic loss and penalty function is given by
$u^*(t) = \sqrt{{s(t)}/{q}} \, r(t)$.
\end{theorem}
The optimal intensity only depends on the position of the most recent post by user $i$ in her follower'{}s feed and thus allows for a very
efficient procedure to sample posting times, which exploits the superposition theorem~\cite{Kingman1992}.
The key idea is as follows: at any given time $t$, we can view the process defined by the optimal intensity as a superposition of $r(t)$
inhomogeneous poisson processes with intensity $\sqrt{{s(t)}/{q}} \, r(t)$ which starts at jumps of the rank $r(t)$, and find the next sample by
computing the minimum across all samples from these processes.
Algorithm~\ref{alg:sampling} summarizes our (sampling) method, which we name \redqueen~\cite{carroll1917through}.
%
Within the algorithm, $othersNextPost(\,)$ returns the time of the next event by other broadcasters in the followers'{} feeds, once the events happens. In practice, we only need to know if the event happens before we post.
Remarkably, it only needs to sample $M(t_f)$ times from a (exponential) distribution (if significance is constant) and requires $O(1)$ space.
\vspace{2mm}

\xhdr{Optimizing for multiple followers}
Given a broadcaster $i$ with $N_i(t) = N(t)$ and $\mu_i(t) = u(t)$ and her followers $\Ncal(i)$ with $\bm{M}_{{\scriptscriptstyle\setminus} i}(t) = \bm{M}(t)$ and
$\gammab_{{\scriptscriptstyle\setminus} i}(t) = \lambdab(t)$, the dynamics of $\bm{M}(t)$ and $\bm{r}(t)$, which we need to solve Eq.~\ref{eq:average-loss}, are given
by:

\begin{align}
	d\bm{r}(t) &= -\bm{r}(t)\,dN(t) + d{\bm{M}}(t) \nonumber \\
	d\bm{\lambda}(t) &= \left[\bm{\lambda}_0'(t)+\bm{w} \odot \bm{\lambda}_0(t) -\bm{w}\odot\bm{\lambda}(t)\right]dt + \bm{\alpha} \odot d\bm{M}(t), \nonumber
\end{align}
where $\odot$ is the element-wise product and $\bm{\alpha}=[\alpha_1,\cdots,\alpha_n]^T$ and $\bm{w}=[w_1,\cdots,w_n]^T$ are the parameters
defining each of the followers'{} feed dynamics, and $n=|\Ncal(i)|$ is the number of followers.

Consider the following quadratic forms for the loss $\ell$ and the penalty $\phi$:
\begin{align}
	\phi(\bm{r}(t_f)) &= \sum_{i=1}^n \frac{1}{2} r_i^2(t_f) \nonumber \\
	\ell(\bm{r}(t),u(t),t) &= \sum_{i=1}^n \frac{1}{2} s_i(t) r_i^2(t) + \frac{1}{2} q \,u^2(t). \nonumber
\end{align}
where $s_i(t)$ is the time significance function for follower $i$, as defined above, and $q$ is a given parameter.
Then, proceeding similarly as in the case of one follower, we can show that: 
%
%
\begin{align}\label{eq:optimal-rate-multi}
	u^*(t) = \sum_{i=1}^{n} \sqrt{{s_i(t)}/{q}} \, r_i(t),
\end{align}
which only depends on the position of the most recent post by user $i$ in her followers'{} feeds. Finally, we can readily adapt \redqueen (Algorithm~\ref{alg:sampling}) to
efficiently sample the posting times using the above intensity -- it only needs to sample $|\cup_{j \in \Ncal(i)} \Fcal_{j \backslash i}(t_f)|$ values and requires $O(|\Ncal(i)|)$ space.

%
%

\vspace{2mm}
\xhdr{Optimizing with an oracle}
In this section, we consider a broadcaster $i$ with $N_i(t) = N(t)$ and $\mu_i(t) = u(t)$, only one of her followers $j$ with $M_{j {\scriptscriptstyle\setminus} i}(t)=M(t)$,
and a constant significance $s(t) = s$. The derivation can be easily adapted to the case of multiple followers and time-varying significance.

Suppose there is an (idealized) oracle that reveals $M(t)$ from $t_0$ to $t_f$, \ie, the history $\Fcal_{j {\scriptscriptstyle\setminus} i}(t_f) = \Fcal(t_f)$ is given, and $M(t_f) = |\Fcal(t_f)| = m$.
Then, we can rewrite Eq.~\ref{eq:average-loss} as 
\begin{align}
\underset{u(t_0, t_f]}{\text{minimize}} & \quad \EE_{N(t_0, t_f]}\left[ \phi(r(t_f)) + 	\int_{t_0}^{t_f} \ell(r(\tau),u(\tau)) \, d\tau \right] \nonumber \\*
\text{subject to} & \quad u(t) \geq 0 \quad \forall t \in (t_0, t_f], \nonumber
\end{align}
where the expectation is only taken over all possible realizations of the counting process $N(t_0, t_f]$ since $M(t_0, t_f]$ is revealed by the oracle and thus
deterministic.
%
%
\begin{figure}[t]
	\centering
	\adjustbox{trim={0.25\width} 0 {0.25\width} 0,clip}{\includegraphics[width=1.2\textwidth]{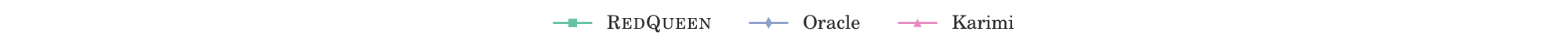}} \\
	\subfloat[\normalsize Position over time]{ 
	\includegraphics[width=.3\textwidth]{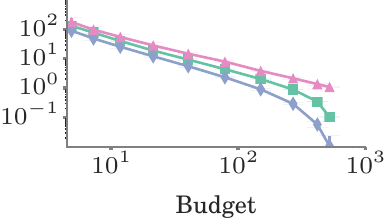}} \hspace{1mm}
	\subfloat[\normalsize Time at the top]{ 
	\includegraphics[width=.3\textwidth,]{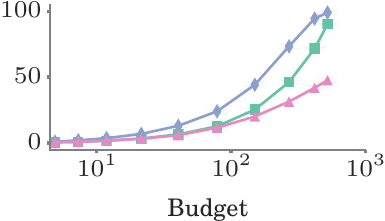}}
	\caption{Optimizing for one follower. Performance of \redqueen in comparison with the oracle and the method by Karimi et al.~\cite{karimi2016smart}
	against number of broadcasted events. The feeds counting processes $M(t)$ due to other broadcasters are Hawkes processes with $\lambda_0=10$,
	$\alpha=1$ and $w=10$. 	In all cases, the time horizon $t_f - t_0$ is chosen such that the number of stories posted by other broadcasters is $\sim$$1000$.
	Error bars are too small to be seen.}
	\label{fig:performance-synthetic-one-follower}
\end{figure}

Similar to the previous sections, assume the loss $\ell$ and penalty $\phi$ are quadratic. It is easy to realize that the best times for user $i$ to post will always coincide with one
of the times in $\Fcal(t_f)$. More specifically, given a posting time $\tau_i \in (t_k, t_{k+1})$, where $t_{k}, t_{k+1} \in \Fcal(t_f)$, one can reduce the cost by $(1/2)q(\tau_i-t_k)r^2(t_k)$ by
choosing instead to post at $t_k$.
As a consequence, we can discretize the  dynamics of $r(t)$ in times $\Fcal(t_f)$, and write $r_{k+1} = r_{k}+1 - (r_{k} + 1) u_k$,
%
%
where $r_k = r(t_k^{\scriptscriptstyle-})$, $u_k = u(t_k^{\scriptscriptstyle+}) \in \{0,1\}$,  $t_k \in \Fcal(t_f)$. We can easily see that $r_k$ is bounded by $0 \leq r_k < r_0+m$. Similarly, we can derive the optimal cost-to-go in
discrete-time as:
%
\begin{align*}
	J(r_k,k) & = \min_{u_k,\ldots,u_m} \frac{1}{2}r_{m+1}^2 + \sum_{i=k}^{m} \frac{1}{2} q \, w_{i+1}\,r_{i+1}^2 + \frac{1}{2} s\,u_i^2,
\end{align*}
where $w_i = t_{i} - t_{i-1}$. 
Next, we can break the minimization and use Bellman'{}s principle of optimality,
\begin{align*}
	J(r_k,k) = \min_{u_k}  \frac{1}{2} q \, w_{k+1} \, r_{k+1}^2 + \frac{1}{2}s\,u_k^2+J(r_{k+1},k+1),
\end{align*}
and, since $u_k \in \{0, 1\}$, the above recursive equation can be written as
\begin{equation}
J(r_k,k) = \min \left\{\frac{1}{2} s + J(0,k+1)  ,\, \frac{1}{2}q\,w_{k+1}(r_{k}+1)^2 + J(r_{k}+1,k+1) \right\}. \nonumber
\end{equation}
Finally, we can find the optimal control $u^{*}_{k},\, k=0, \ldots, m$ and cost $J(r_0,0)$ by backtracking from the terminal condition $J(r_{m+1},m+1)=r_{m+1}^2/2$
to the initial state $r_0$, as summarized in Algorithm~\ref{alg:oracle}, which can be adapted to multiple followers.
Note that, in this case, the optimal strategy is not stochastic and consists of a set of optimal posting times, as one could have guessed.
However, for multiple followers, the complexity of the algorithm is $O(m^{2})$, where
$m = |\cup_{j \in \Ncal(i)} \Fcal_{j \backslash i}(t_f)|$.
%




%% file: 050experiments.tex
\begin{figure}[t]
      \centering
      \adjustbox{trim={0.25\width} 0 {0.25\width} 0,clip}{\includegraphics[width=1.2\textwidth]{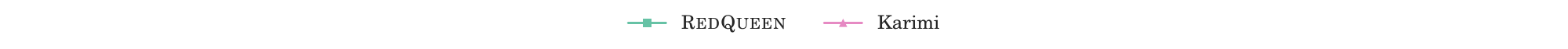}} \\
      \begin{tabular}{cc}
      \includegraphics[width=.3\textwidth]{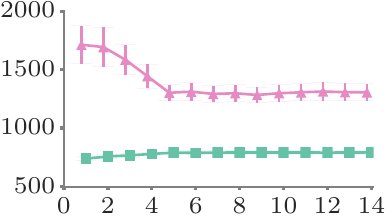} &
      {\includegraphics[width=.3\textwidth]{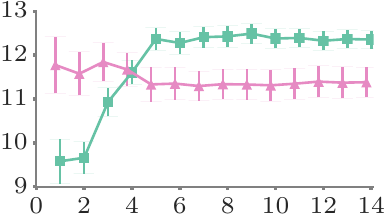}} \\
      {(a) Position over time} & {(b) Time at the top}
      \end{tabular} \label{fig:performance-synthetic-multiple-followers-number}
      \caption{Optimizing for multiple followers. Performance of \redqueen in comparison with the method by Karimi et al.~\cite{karimi2016smart} against number of followers.
      The feeds counting processes $M(t)$ due to other broadcasters follow piecewise constant intensities, where the intensity of each follower remains constant within each piece,
      it varies as a half-sinusoid across pieces and it starts with a random initial phase. The performance of both methods stays constant upon addition of more followers.
      %
      }
	\label{fig:performance-synthetic-multiple-followers}
\end{figure}

\subsection{Experiments on synthetic data}

\xhdr{Experimental setup}
We evaluate the performance via two quality measures: position over time, ${\scriptstyle \int_{0}^{T} r(t) dt}$, 
and time at the top, ${\scriptstyle \int_{0}^{T} \II(r(t) < 1) dt}$ and compare the performance of \redqueen against the oracle, described in Section~\ref{sec:method}, 
and the method by Karimi et al.~\cite{karimi2016smart}, which, to the best of our knowledge, is the state of the art.
Unless otherwise stated, we set the significance $s_i(t) = 1,\, \forall\, t, i$ and use the parameter $q$ to control the number of posts by \redqueen\footnote{\scriptsize The expected number of posts by \redqueen{} are a decreasing function of $q$. Hence, we can use binary search to guess $q$ and then use averaging over multiple simulation runs to estimate the number of posts made.}.

\vspace{2mm}
\xhdr{Optimizing for one follower}
We first experiment with one broadcaster and one follower against an increasing number of events (or budget).
We generate the counting processes $M(t)$ due to other broadcasters using Hawkes processes, which are particular instances of
the general functional form given by Eq.~\ref{eq:intensity}. We perform {$10$ independent simulation runs} and compute the average and standard error (or standard deviation) of the quality measures.
Fig.~\ref{fig:performance-synthetic-one-follower} summarizes the results, which show that our method: (i) consistently outperforms the method by Karimi et al. by large margins;
(ii) achieves at most $3$$\times$ higher position over time than the oracle as long as the budget is $<$$30$\% of the posted events by all other broadcasters; and, (iii)
achieves $>$$40$\% of the value of time at the top that the oracle achieves.

\vspace{2mm}
\xhdr{Optimizing for multiple followers}
Next, we ex\-pe\-ri\-ment with one broadcaster and multiple followers.
In this case, we generate the counting processes $M(t)$ due to other broadcasters using piece-constant intensity functions. More specifically,
we simulate the feeds of each follower for $1$ day, using $24$ $1$-hour long segments, where the rate of posts remains constant per follower 
in each segment and the rate itself varies as a half-sinusoid (\ie, from $\sin{0}$ to $\sin{\pi}$), with each follower starting with a random initial phase.
This experimental setup reproduces volume changes throughout the day across followers'{} feeds in different time-zones and closely resembles the 
settings in previous work~\cite{karimi2016smart}.
The total number of posts by the \redqueen broadcaster is kept nearly constant and is used as the budget for the other baselines. Additionally, for Karimi'{}s 
method, we provide as input the true empirical rate of tweets per hour for each user. Here, we do not compare with the oracle since, due to its quadratic
complexity, it does not scale.


Figure~\ref{fig:performance-synthetic-multiple-followers} summarizes the results. In terms of position over time, \redqueen outperforms Karimi'{}s method by a factor of 
$2$.
In terms of time at the top, \redqueen achieves $\sim$18\% lower values than Karimi'{}s method for $1$-$4$ followers but $\sim$10\% higher values for $>$$5$ followers.
A potential reason for Karimi'{}s method to performs best in terms of time at the top for a low number of followers and piecewise constant intensities is that, while the number of 
followers is low, there are segments which are clearly favorable and thus Karimi'{}s method concentrates posts on those, however, as the number of followers increases, there
are no clear favorable segments and thus advance planning does not give Karimi'{}s method any advantage. 
On the other hand, \redqueen, due to its online nature, is able to adapt to transient variations in the feeds. 
%

\subsection{Experiments on real data}

\begin{figure}[t]
	\centering
        \begin{tabular}{cc}
        {\includegraphics[width=.3\textwidth]{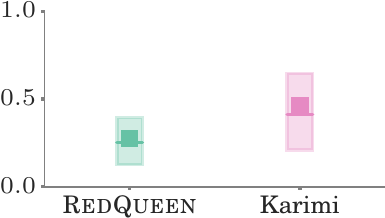}} & {\includegraphics[width=.3\textwidth]{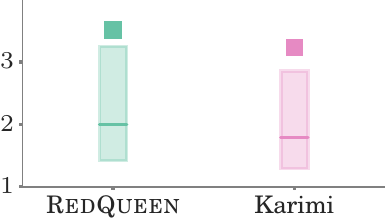}}\\
        {(a) Position over time} & {(b) Time at the top}
        \end{tabular}
	\caption{Performance of \redqueen and the method by Karimi et al.~\cite{karimi2016smart} for $2000$ Twitter users, picked
	at random. The solid horizontal line (square) shows the median (mean) quality measure, normalized with respect to the value achieved by the users'{} actual
	\emph{true} posts, and the box limits correspond to the 25\%-75\% percentiles.}
	\label{fig:performance-real}
\end{figure}

\xhdr{Dataset description and experimental setup}
We use data gathered from Twitter as reported in previous work~\cite{cha2010measuring}, which comprises profiles of $52$ million
users, $1.9$ billion directed follow links among these users, and $1.7$ billion public tweets posted by the collected users. The follow
link information is based on a snapshot taken at the time of data collection, in September 2009.
Here, we focus on the tweets published during a two month period, from July 1, 2009 to September 1, 2009, in order to be able to consider
the social graph to be approximately static, and sample $2000$ users uniformly at random as broadcasters and record all the tweets
they posted.
Then, for each of these broadcasters, we track down their followers and record all the (re)tweets they posted as well as reconstruct
their timelines by collecting all the (re)tweets published by the people they follow.
We assign equal significance to each follower but filter out those who follow more
than $500$ people since, otherwise, they would dominate the optimal strategy.
Finally, we tune $q$ such that the total number of tweets posted by our method is equal to the number of tweets the broadcasters
tweeted during the two month period (with a tolerance of $10\%$).

\xhdr{Solution quality}
We only compare the performance of our method against the method by Karimi et al.~\cite{karimi2016smart} since the oracle does not scale to the size of real data.
Moreover, for the method by Karimi et al., we divide the two month period into ten segments of approximately one week to fit the piecewise constant intensities of the
followers'{} timelines, which the method requires.
Fig.~\ref{fig:performance-real} summarizes the results by means of box plots, where position over time and time at the top are normalized with respect to the value
achieved by the broadcasters'{} actual \emph{true} posts during the two month period.
That means, if $y=1$, the optimized intensity achieves the same position over time or time at the top as the broadcaster's{} true posts.
In terms of position over time and time at the top, \redqueen consistently outperforms competing methods by large margins and achieves $0.28$$\times$ lower average position
and $3.5$$\times$ higher time at the top, in average, than the broadcasters'{} true posts -- in fact, it achieves lower position over time (higher time at the top) for
$100$\% ($99.1$\%) of the users.

\xhdr{Time significance} We look at the actual broadcasting strategies for one real user and investigate the effect of a time varying
significance. We define $s_i(t)$ to be the probability that follower $i$ is online on that weekday, estimated empirically using the (re)tweets the follower posted as in
Karimi et al.~\cite{karimi2016smart}.
Fig.~\ref{fig:individual-user-real} compares the position over time for the most recent tweet posted by a real user against the most recent one \emph{posted} by a
simulation run of \redqueen with and without time varying significance. We can see that without significance information, \redqueen posts at nearly an even pace.
However, when we supply empirically estimated significance, \redqueen avoids tweeting at times the followers are unlikely to be active, \ie, the weekends, denoted by the shaded areas in panel (c) of Fig.~\ref{fig:individual-user-real}.
Due to this, the average position (maximum position) falls from $389{.}45$ ($1085{.}17$) to $425{.}25$ ($1431{.}0$), but is still lower than $698{.}04$ ($2597{.}9$) 
obtained by the user's original posting schedule.
\begin{figure*}[t!]
	\centering
	\begin{tabular}{cc}
        \includegraphics[width=.3\textwidth]{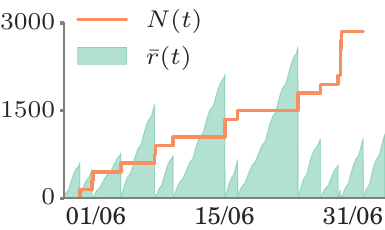} &
	\includegraphics[width=.3\textwidth]{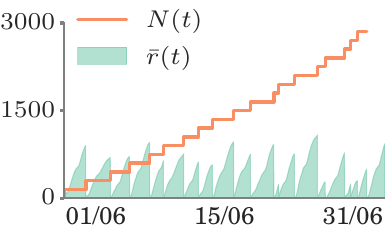} \\
	 {(a) True posts} & {(b) \redqueen (without significance)} \\
	 {$\frac{1}{T} \int_{0}^{T} \bar{r}(t) dt = 698.04$} & {$\frac{1}{T} \int_{0}^{T} \bar{r}(t) dt = 389.45$} \\
	 \vspace{2mm} {} \\
	\includegraphics[width=.3\textwidth]{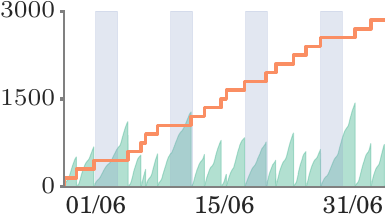} &
	\includegraphics[width=.3\textwidth]{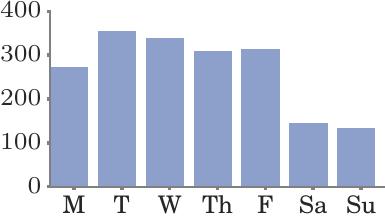} \\
	{(c) \redqueen (with significance)}  & {(d) Followers' (re)tweets per weekday}\\
  	 {$\frac{1}{T} \int_{0}^{T} \bar{r}(t) dt = 425.25$} & {}
	 \end{tabular}
	\caption{A broadcaster chosen from real data. Panels compares the position over time $\bar{r}(t) = \sum_{i=0}^{\Ncal} r(t) / \Ncal$ (in green; lower is better) for the most recent tweet posted by a real user against the most
	recent one \emph{posted} by a simulation run of \redqueen without and with significance. Here, the orange staircases represent the counts $N(t)$ of the tweets posted by the real user and \redqueen over time. The shaded area in panel (c) highlights weekends. We can see that \redqueen avoided tweeting on weekends, when the followers are less likely to be active/logged-in, as seen in panel (d).}
	\label{fig:individual-user-real}
\end{figure*}


%
%
%

%% file: 060conclusions.tex
In this paper, we approached the when-to-post problem from the perspective of stochastic optimal control and showed that the optimal
broadcasting strategy is surprisingly simple -- it is given by the position of her most recent post on each of her follower'{}s feed. Such
a strategy can be implemented using a simple and efficient on-line algorithm.
We experimented with synthetic and real-world data gathered from Twitter and showed that our algorithm consistently makes a user'{}s
posts more visible over time and it significantly outperforms the state of the art.

Our work also opens many venues for future work. For example, in this work, we considered social networks that sort stories in the users'{}
feeds in inverse chronological order (\eg, Twitter, Weibo). Extending our methodology to social networks that sort stories algorithmically
(\eg, Facebook) is a natural next step.
Currently, \redqueen optimizes a quadratic loss on the position of a broadcaster'{}s most recent post on her followers'{} feeds over time. However, it would be useful to derive optimal broadcasting intensities for other losses, \eg, time at the top.
Moreover, we assume that only one broadcaster is using \redqueen. A very interesting follow-up would be augmenting our
framework to consider multiple broadcasters under cooperative, competitive and adversarial environments.
Finally, the novel technical aspects of our problem formulation, \eg, optimal control of jump SDEs with double stochastic temporal
point processes, can be applied to other control problems in social and information networks such as activity shaping~\cite{shaping14nips} and
opinion control~\cite{wang2016steering}.

%% file: 070appendix.tex
\subsection*{Proof of Proposition~1} \label{app:jsde-intensity} 
%
%
Using the left continuity of poisson processes, we have that:
\begin{align*}
	d\lambda^{*}(t)&=\lambda_0'(t)dt + \alpha \, d\left(\int_0^t g(t-s) dN(s) \right).
\end{align*}
Then, using Ito'{}s calculus~\cite{hanson2007}, we can rewrite the differential in the
second term as
\begin{align*}
	\int_0^{t+dt} g(t+dt-s)\,dN(s) - \int_0^t g(t-s)\,dN(s) &=\int_0^{t+dt} (g(t-s)+g'(t-s)dt) \, dN(s) - \int_0^t g(t-s)\,dN(s) \\
	&=\int_t^{t+dt} g(t-s)\,dN(s) + dt\int_0^{t+dt} g'(t-s)\,dN(s) \\
	&=g(0)dN(t) -w \, dt\int_0^{t+dt} g(t-s)\,dN(s)	 \\
	&=dN(t) - w \, dt \int_0^{t} g(t-s)\,dN(s)		 \\
	&=dN(t)+ \frac{w}{\alpha} [\lambda_0(t)-\lambda^{*}(t)]dt.
\end{align*}
This completes the proof.

\subsection*{Proof of Lemma~3} \label{app:bellman} 
\begin{align*}
	\min_{u(t,t_f]} \mathbb{E}_{(N,M)(t,t_f]} \left[ \phi(r(t_f)) + \int_t^{t_f} \ell(r(\tau),u(\tau)) \, d\tau \right ] &= \min_{u(t,t_f]} \mathbb{E}_{(N,M)(t,t_f]} \bigg[ \phi(r(t_f)) + \int_t^{t+dt} \ell(r(\tau),u(\tau)) \, d\tau \nonumber \\ & \hspace{42mm} + \int_{t+dt}^{t_f} \ell(r(\tau),u(\tau)) \, d\tau \bigg ] \nonumber \\
	&= \min_{u(t,t_f]} \mathbb{E}_{(N,M)(t,t+dt]}\bigg[\mathbb{E}_{(N,M)(t+dt,t_f]} \Big[ \phi(r(t_f)) + \ell(t,r,u) \, dt  \nonumber \\ & \hspace{42mm} + \int_{t+dt}^{t_f} \ell(r(\tau),u(\tau)) \, d\tau \Big ] \bigg] \\
	&= \min_{u(t,t+dt]} \min_{u(t+dt,t_f]} \mathbb{E}_{(N,M)(t,t+dt]}\bigg[\ell(r(t),\lambda(t),t) \, dt \nonumber \\ & \hspace{10mm}
	+\mathbb{E}_{(N,M)(t+dt,t_f]} \Big[ \phi(r(t_f)) + \int_{t+dt}^{t_f} \ell(r(\tau),u(\tau)) \, d\tau \Big ] \bigg]  \\
	&= \min_{u(t,t+dt]} \mathbb{E}_{(N,M)(t,t+dt]} \left[ J(\lambda(t+dt),r(t+dt),t+dt)\right ] \nonumber \\
	&\hspace{20mm} + \ell(r(t),u(t)) \, dt.
\end{align*}

\subsection*{Proof of Lemma~4} \label{app:opt-con-sol} 
According to the Stone-Weierstrass theorem, any continuous function in a closed interval can be approximated as closely as desired by a polynomial function \cite{stone1948}. So by assuming the continuity of cost function we consider general form
\begin{align*}
	J(r(t),\lambda(t),t) = \sum_{i=0}^n \sum_{j=0}^m f_{ij}(t)r^i(t)\lambda^j(t),
\end{align*}
where $m$ and $n$ are arbitrary large numbers. Indeed in each time $t$ we approximate a two variate function of $r(t)$ and $\lambda(t)$ by a polynomial where the coefficient are defined by the time varying functions $f_{ij}(t)$. If we substitute this function in to Eq. \ref{eq:bellman-pde} and simplifying the expression we would have
\begin{align*}
	0&=\sum_{i=1}^n f'_{i0} \, r^i(t) + f_{i0}\,(r+1)^i \lambda - f_{i0} \, r^i \lambda +\sum_{j=1}^m f'_{0j} \lambda^j + j(\lambda'_0+\beta\lambda_0-\beta\lambda)f_{0j} \lambda^{j-1} +f_{0j} (\lambda+\alpha)^j \lambda - f_{0j} \lambda^{j+1} \\
	&+\sum_{i=1}^n \sum_{j=1}^m j(\lambda'_0+\beta\lambda_0-\beta\lambda)f_{ij} r^i \lambda^{j-1}+\sum_{i=1}^n \sum_{j=1}^m f_{ij} (r+1)^i (\lambda+\alpha)^j\lambda - f_{ij} r^i \lambda^{j+1} \\
	&-\frac{1}{2}s^{-1}\bigg[\sum_{i=1}^n f_{i0}r^i + \sum_{i=1}^n \sum_{j=1}^m  f_{ij} r^i \lambda^j \bigg]^2 +\frac{1}{2}q\,r^2 +f_{00}'
\end{align*}
where for notational simplicity we omitted the time argument of functions. To find the unknown functions $f_{ij}(t)$,  we equate the coefficient of different variables. If we consider the coefficient of $r^{2n}$, we have $f_{n0}(t)=0$. We can continue this argument for $n-1,n-2,\cdots,2$ to show that $\forall i\geq 2; \, f_{i0}(t)=0$. Similar reasoning for coefficients of $r^{2i}\lambda^{2j}$ shows that $\forall j, i\geq 2; f_{ij}(t)=0$. Finally, the coefficient of $r^2$ is $1/2q-1/2\,s^{-1}f^2_{10}(t)=0$ so $f_{10}(t)=(sq)^{1/2}$. If we rename $f_{0j}(t)$ to $g_j(t)$ and $f_{00}(t)$ to $f(t)$, then we have
\begin{align*}
	J(r(t),\lambda(t),t) = f(t)+ (sq)^{1/2} r(t) + \sum_{j=1}^m g_{j}(t)\lambda^j(t).
\end{align*}
We can continue the previous method to find the remaining coefficients and completely define the cost-to-go function. If we equate the coefficient of $\lambda^j$ to zero we would have a system of first oder differential equation which its $j$'th row is
\begin{align*}
	g_j'(t) + j(\alpha- \beta) g_j(t) + (j+1)\big(\lambda_0'(t)+\beta\lambda_0(t)+\frac{j}{2}\alpha^2\big)g_{j+1}(t) + \sum_{k=2}^{m-j} \binom{j+k}{k+1} \alpha^{k+1} g_{j+k}(t) = 0
\end{align*}
When $\lambda_0(t)=\lambda_0$, we can express this using matrix differential equation $\bm{g}'(t) = A\bm{g}(t)$.
and its solution is $\bm{g}(t) = c_1 e^{\zeta_1 t} \bm{u}_1 + c_2 e^{\zeta_2 t} \bm{u}_2 + \cdots + c_n e^{\zeta_n t} \bm{u}_n	$
where $\zeta_i$ and $\bm{u}_i$ are eigenvalue and eigenvector of matrix $A$ and $c_i$ is a constant found using the terminal conditions. Since in triangular matrices diagonal entries are eigenvalues, we have
	$ \bm{g}(t) = \sum_{j=1}^m c_i e^{j(\beta-\alpha)} \bm{u}_i $.
We can approximate general time varying $\lambda_0(t)$ using piecewise function and repeat the above procedure for each piece.
\subsection*{Lemma~6}\label{app:differentialSDE} 
%
\begin{lemma}\label{thm:differentialSDE}
	Let $x(t)$ and $y(t)$ be two jump-diffusion processes defined by following jump  SDEs:
	\begin{align}
		dx(t) &= f(x(t),t) dt + h(x(t),t) dN(t) + g(x(t),t) dM(t) \nonumber \\
		dy(t) &= m(y(t),t) dt + n(y(t),t) dM(t), \nonumber
	\end{align}
where $N(t)$, $M(t)$ are independent jump processes.
%
%
If function $F(x,y,t)$ is once continuously differentiable in $x$, $y$ and $t$, then,
\begin{align}
	dF(x(t),y(t),t) &= (F_t + f F_x + m F_y)(x(t),y(t),t) dt +\left[F\big(x(t)+h(x(t),t),\,y(t),\,t\big)-F(x(t),y(t),t)\right] dN(t) \nonumber \\
	& +\left[F\big(x(t)+g(x(t),t),\,y(t)+n(y(t),t),\,t\big)-F(x(t),y(t),t)\right] dM(t). \nonumber
\end{align}
\end{lemma}

\begin{proof}
According to the definition of differential,
\begin{align*}
	dF \coloneqq dF(x(t),y(t),t) &= F(x(t+dt),y(t+dt),t+dt) - F(x(t),y(t),t) \\ %
	&=F\big(x(t)+dx(t),y(t)+dy(t),t+dt\big) - F\big(x(t),y(t),t\big)
\end{align*}
where we used the complete notation for $F$ to be more clear.  Using the zero-one law of point processes, we can write
\begin{align*}
	dF &= F\big(x+fdt+h,\,y+mdt,\,t+dt\big)\,dN(t) + F\big(x+fdt+g,\,y+mdt+n,\,t+dt\big)\,dM(t) \\
	&+ F\big(x+f dt,\,y+m dt,\,t+dt\big)\,(1-dN(t))(1-dM(t)) - F\big(x,y,t\big)
\end{align*}
where for notational simplicity we drop arguments of all functions except $F$.
Then, we can expand the first three terms in the right hand sides:
\begin{align*}
	F\big(x+fdt+h,\,y+mdt,\,t+dt\big) &=F(x+h,y,t)+F_x(x+h,y,t)fdt+F_y(x+h,y,t)mdt +F_t(x+h,y,t)dt \\
	&+F\big(x+fdt+g,\,y+mdt+n,\,t+dt\big) \\
	& = F(x+g,y+n,t) + F_x(x+g,y+n,t)fdt + F_y(x+g,y+n,t)mdt \\
	&+F_t(x+g,y+n,t)dt + F\big(x+fdt,\,y+mdt,\,t+dt\big) \\
	& = F(x,y,t) +F_x(x,y,t)fdt + F_y(x,y,t)mdt+F_t(x,y,t)dt,
\end{align*}
using that the bilinear differential form $dt \, dN(t)=0$~\cite{hanson2007} and $dN(t)dM(t)=0$ by the zero-one jump law~\cite{Kingman1992}. Finally
\begin{align*}
	dF = (fF_x+mF_y+F_t)(x,y,t)dt + \big[ F(x+h,y,t)-F(x,y,t)) \big]dN(t)
	+\big[ F(x+g,y+n,t) - F(x,y,t) \big]dM(t).
\end{align*}
\end{proof}

\newpage